\documentclass{svproc}

\usepackage{multicol}
\usepackage{float}
\usepackage{amsfonts}
\usepackage{comment}
\usepackage{times}
\usepackage{amsmath}
\usepackage{changepage}
\usepackage{amssymb}
\usepackage{graphicx}
\usepackage{adjustbox}
\usepackage{latexsym}
\usepackage{enumerate}
\usepackage[font=small,belowskip=0pt]{caption}
\usepackage{mathtools}
\usepackage{algorithm}
\usepackage{algcompatible}
\usepackage{algpseudocode}
\usepackage{bbm}
\usepackage{xcolor}
\usepackage{bm}
\usepackage{booktabs}

\definecolor{porange}{HTML}{E77500} 
\algrenewcommand{\algorithmiccomment}[1]{\textcolor{porange}{\hfill// #1}}
\algnewcommand{\LineComment}[1]{\Statex \textbf{\textcolor{porange}{// #1}}}

\usepackage{url}
\usepackage{hyperref}
 \hypersetup{
     colorlinks=true,
     linkcolor=blue,
     filecolor=blue,
     citecolor=blue,      
     urlcolor=magenta,
     }

\usepackage[%
    style=numeric-comp,
    sorting=none,
    backend=biber,
    sortcites=true,
    doi=false,
    url=false,
    firstinits=true,
    hyperref,
    isbn=false,
    eprint=false,
    minbibnames=1,   
    maxbibnames=3,
    block=none]
    {biblatex}

\AtEveryBibitem{\clearfield{pages}}

\renewbibmacro{in:}{}

\newbibmacro{string+doiurl}[1]{
    \iffieldundef{doi}{
        \iffieldundef{url}{#1}{\href{\thefield{url}}{#1} }
    }{\href{https://dx.doi.org/\thefield{doi}}{#1}}
}
\DeclareFieldFormat{title}{\usebibmacro{string+doiurl}{\mkbibemph{#1}}}
\DeclareFieldFormat[article,incollection,techreport,inproceedings,book,misc]%
    {title}%
    {\usebibmacro{string+doiurl}{\mkbibquote{#1}}}

\usepackage{amsmath}
\usepackage{bm}

\usepackage{ifthen}
\newboolean{include-notes}
\newboolean{include-new}
\newboolean{include-remove}
\setboolean{include-notes}{false}
\setboolean{include-new}{false}
\setboolean{include-remove}{false}

\usepackage{xcolor}
\usepackage[normalem]{ulem}
\newcommand{\jaime}[1]{\ifthenelse{\boolean{include-notes}}{\textcolor{orange}{\textbf{Jaime:} #1}}{}}
\newcommand{\haimin}[1]{\ifthenelse{\boolean{include-notes}}{\textcolor{teal}{\textbf{Haimin:} #1}}{}}
\newcommand{\remove}[1]{\ifthenelse{\boolean{include-remove}}{\textcolor{red}{\sout{#1}}}{}}
\newcommand{\new}[1]{\ifthenelse{\boolean{include-new}}{\textcolor{magenta}{#1}}{#1}}

\usepackage{lipsum}

\newcommand{\reals}{\mathbb{R}}

\newcommand{\distr}{p}
\newcommand{\prob}{P}
\newcommand{\mean}{\mu}

\newcommand{\covar}{\Sigma}
\newcommand{\entropy}{H}
\newcommand{\gaussian}{{\mathcal{N}}}

\newcommand{\ivec}{\mathcal{I}}
\newcommand{\bel}{b}

\DeclareMathOperator*{\expectation}{\mathbb{E}}

\DeclareMathOperator*{\argmax}{arg\,max}

\DeclareMathOperator*{\subjectto}{subject\,\, to \quad}

\newcommand{\state}{{x}}
\newcommand{\ctrl}{{u}}
\newcommand{\dstb}{{d}}

\newcommand{\cset}{{\mathcal{U}}}

\newcommand{\dyn}{{f}}

\newcommand{\valfunc}{{V}}
\newcommand{\qfunc}{{Q}}

\newcommand{\policy}{{\pi}}

\newcommand{\failure}{{\mathcal{F}}}

\renewcommand\algorithmicfunction{\textbf{function}}
\renewcommand\Function{\item[ \algorithmicfunction]}

\newcommand{\mset}{{\mathcal{M}}}

\newcommand{\nodeset}{\mathcal{N}}
\newcommand{\leafset}{\mathcal{L}}
\newcommand{\cnodeset}{\mathcal{C}}

\newcommand{\node}{n}
\newcommand{\tnode}{{\tilde{\node}}}
\newcommand{\pre}[1]{\operatorname{\mathfrak{p}}(#1)}

\newcommand{\cl}{\text{cl}}
\newcommand{\Sim}{\text{sim}}

\bibliography{references}

\usepackage{url}

\begin{document}
\mainmatter              
\title{Active Uncertainty Reduction for Human-Robot Interaction: An Implicit Dual Control Approach}
\titlerunning{Active Uncertainty Reduction for HRI}  
%
\author{Haimin Hu\inst{1} \and Jaime F. Fisac\inst{1}}
\authorrunning{Haimin Hu et al.} 
%
\tocauthor{Haimin Hu, and Jaime F. Fisac}
\institute{Department of Electrical and Computer Engineering, Princeton University, USA,\\
\email{\{haiminh,jfisac\}@princeton.edu}}

\maketitle              

\begin{abstract}
\new{The ability to accurately predict human behavior is central to the safety and efficiency of robot autonomy in interactive settings.}
\new{Unfortunately, robots often lack access to key information on which these predictions may hinge, such as people's goals, attention, and willingness to cooperate.}
Dual control theory addresses this challenge by treating unknown parameters \new{of a predictive model} as stochastic hidden states and \remove{identifying}\new{inferring} their values \new{at runtime} using information gathered during \remove{control of the robot}\new{system operation}.
\remove{Despite its ability}\new{While able} to optimally and automatically trade off exploration and exploitation, dual control is computationally intractable for general \remove{human-in-the-loop}\new{interactive} motion planning, mainly due to \remove{nested}\new{the fundamental coupling between robot} trajectory optimization and human intent \remove{prediction}\new{inference}.
In this paper, we present a novel algorithmic approach to enable active uncertainty reduction for \remove{human-in-the-loop}\new{interactive} motion planning based on the implicit dual control paradigm.
Our approach relies on sampling-based approximation of stochastic dynamic programming, leading to a model predictive control problem that can be readily solved by real-time gradient-based optimization methods.
The resulting policy is shown to preserve the dual control effect for \remove{generic}\new{a broad class of} predictive human models with both continuous and categorical uncertainty.
The efficacy of our approach is demonstrated with simulated driving examples.

\keywords{human-robot interaction, dual control theory, stochastic MPC}
\end{abstract}

\section{Introduction}
\remove{Human-in-the-loop robot planning}\new{Computing robot plans that account for possible interactions with one or multiple humans} \remove{can prove challenging}\new{is a challenging task}, as the robotic system and human agents may have coupled dynamics, limited communication capabilities, and conflicting interests.
To achieve safety and efficiency in human-robot interaction settings, the robot must competently predict and seamlessly adapt to human behavior.
\new{Intent-driven behavior models are widely used for such predictions: for example, the} Boltzmann model of noisily rational human \remove{behavior}\new{decision-making}~\cite{luce1959individual,ziebart2008maximum}\remove{is commonly used for human motion prediction, which} assumes that the human is exponentially likelier to \remove{pick}\new{take} actions \remove{that maximize an}\new{with a higher} underlying \new{\emph{utility}}\remove{state-action value function}.
If the human's intent is well captured by a given \remove{value}\new{utility} function, the interaction can be modeled as a dynamic game in which \remove{both }the human's and the robot's \remove{closed-loop}\new{feedback} strategies can be obtained via dynamic programming~\cite{fisac2019hierarchical}.
However, typical interaction settings may admit a plethora of \new{\emph{a priori}} plausible \remove{candidate value functions}\new{intents} (e.g., corresponding to \remove{different}\new{distinct} equilibrium solutions~\cite{peters2020inference} \new{or different human preferences~\cite{sadigh2018planning}}), \remove{whose parameters are not necessarily observable to}\new{which in general cannot be fully modeled, let alone observed, by} the robot\new{~\cite{fisacBHFWTD18}}.
The robot may seek to \new{represent the human's intent through a parametric model and then }infer \new{the value of} these \new{parameters} as hidden states under a Bayesian framework~\cite{fisacBHFWTD18,tian2021safety,hu2022sharp}, but doing so tractably while planning through interactions is an open problem.

\looseness=-1
Multi-stage trajectory optimization with closed-loop Bayesian inference can be generally cast as a stochastic optimal control problem.
An important aspect of stochastic control with hidden states is whether the computed policy generates the so-called \emph{dual control} effect~\cite{feldbaum1960dual,bar1974dual,mesbah2018stochastic}; that is, in the context of human-robot interaction, whether the robot \emph{actively} seeks to reduce the uncertainty about human's hidden states.
Solution methods for dual stochastic optimal control problems can be categorized into \emph{explicit} approaches~\cite{heirung2015mpc,sadigh2018planning,tian2021anytime}, which reformulate the problem with some form of
heuristic probing, and implicit approaches~\cite{bar1974dual,klenske2016dual,arcari2020dual}, which directly tackle the control problem with stochastic dynamic programming.
While explicit dual control problems are in general easier to formulate and solve than their implicit counterparts, designing the probing term and tuning its weighting factor can be non-trivial and may lead to inconsistent performance.
For a comprehensive review of dual control methods we refer to~\cite{mesbah2018stochastic}.

\vspace{0.2cm}
\looseness=-1
\noindent \textbf{Contribution:} In this paper, we formulate a broad class of \remove{human-in-the-loop}\new{interactive} planning problems in the framework of stochastic optimal control and present an approximate solution method using implicit dual stochastic model predictive control (SMPC).
The resulting policy automatically trades off the cost of exploration and exploitation, allowing the robot to actively reduce the uncertainty about the human's hidden states without sacrificing expected planning performance.
Our proposed SMPC problem supports both continuous and categorical human uncertainty and can be solved using off-the-shelf real-time nonlinear optimization solvers.
To \remove{the best of }our knowledge, this is the first \remove{human-in-the-loop}\new{human-robot} interactive \new{motion} planning framework \remove{with}\new{that performs} active uncertainty reduction \remove{that does not require designing}\new{without requiring} an explicit information gathering strategy or objective.

\section{Related Work}
\label{sec:related_work}

\noindent \textbf{Human-Robot Interaction as a POMDP.}
Robotic motion planning problems that involve identification of human intentions and behaviors can be modeled as a Mixed Observability Markov Decision Process~\cite{bandyopadhyay2013intention}, a variant of the well-known Partially Observable Markov Decision Process (POMDP).
Although intractable in the general form, efficient algorithms such as DESPOT~\cite{somani2013despot} and POMCP~\cite{silver2010monte} have been developed to approximately solve POMDPs.
In~\cite{sehr2017tractable}, a dual SMPC method is proposed to solve moderate-sized POMDPs.
The resulting policy naturally exhibits system probing and automatically trades off exploration and exploitation.
When interactions between the robot and the human are modeled, the POMDP formulation becomes a (usually intractable) Partially-Observable Stochastic Game (POSG)~\cite{sadigh2018planning}.
Our approach can be viewed as a new computationally efficient framework for solving human-robot interaction problems cast as POSGs.


\noindent \textbf{Stochastic Model Predictive Control.} 
SMPC has been widely used for robotic motion planning under uncertainty due to its ability to handle safety-critical constraints and general uncertainty models.
In~\cite{schildbach2015scenario}, an SMPC approach is proposed for lane change assistance of the ego vehicle in the presence of other human-driven vehicles.
In~\cite{chen2021interactive}, a scenario-based SMPC algorithm is proposed to capture multimodal reactive behaviors of uncontrolled human agents.
In~\cite{hu2022sharp}, a provably safe SMPC planner is developed for human-aware robot planning, which proactively balances expected performance with the risk of high-cost emergency safety maneuvers triggered by low-probability human behaviors.
However, all those SMPC methods do not produce dual control effect.
In~\cite{arcari2020dual}, an implicit dual SMPC is proposed for optimal control of nonlinear dynamical systems with both parametric and structural uncertainty.
Our paper builds on this approach to enable active uncertainty reduction for human-robot interaction.

\noindent \textbf{Active Information Gathering.} 
To date, most human-in-the-loop planning methods follow a ``passively adaptive'' paradigm.
In~\cite{peters2020inference}, human-robot interaction is modeled as a general-sum differential game with human's equilibrium uncertainty.
The robot first infers which equilibrium the humans are operating at and then aligns its own strategy with the inferred human's equilibrium solution.
Recently, the notion of active information gathering, which is conceptually very similar to dual control, has received attention from the human-robot interaction community.
In~\cite{sadigh2018planning}, an additional information gathering term is added to the robot's nominal objective in a trajectory optimization framework for online estimating unknown parameters of the human's state-action value function.
In fact, according to the categorization proposed by~\cite{mesbah2018stochastic}, the planning formulation in~\cite{sadigh2018planning} can be classified as an explicit dual control approach, which requires heuristic design of a probing mechanism and weighing the relative importance between optimizing the expected performance and reducing the uncertainty of human's unknown parameters.
On the contrary, we propose an implicit dual control approach, which automatically balances performance with uncertainty reduction, thus not requiring design of a heuristic information gathering mechanism.

\section{Preliminaries}
\label{sec:prelim}

\looseness=-1
\textbf{Human-Robot Joint System.}
We consider a class of discrete-time input-affine dynamics that capture the interaction between a human individual \emph{or group} and a robotic system,
\begin{equation}
\label{eq:joint_sys}
\state_{t+1} = \dyn (\state_t) + B^R(\state_t) \ctrl^R_t + B^H(\state_t) \ctrl^H_t + \dstb_t,
\end{equation}
where $\state_t \in \reals^n$ is the joint state vector, $\ctrl^R_t \in \cset^R \subseteq \reals^{m_R}$ and $\ctrl^H_t \in \cset^H \subseteq \reals^{m_H}$ are the control vectors of the robot and the human, respectively, $f: \reals^{n} \rightarrow \reals^{n}$ is a nonlinear function that describes the autonomous part of the dynamics, $B^R: \reals^n \rightarrow \reals^{n \times m_R}$ and $B^H: \reals^n \rightarrow \reals^{n \times m_H}$ are control input matrices that can depend on the state, and $\dstb_t$ is an additive uncertainty term representing external disturbance inputs (e.g. wind) and modeling error. 
For simplicity, we further assume that $\dstb_t \sim \gaussian(0, \covar_\dstb)$ is a zero-mean i.i.d. Gaussian random variable. 

%
%
\begin{figure}[!hbtp]
  \centering
  \includegraphics[width=0.9\columnwidth]{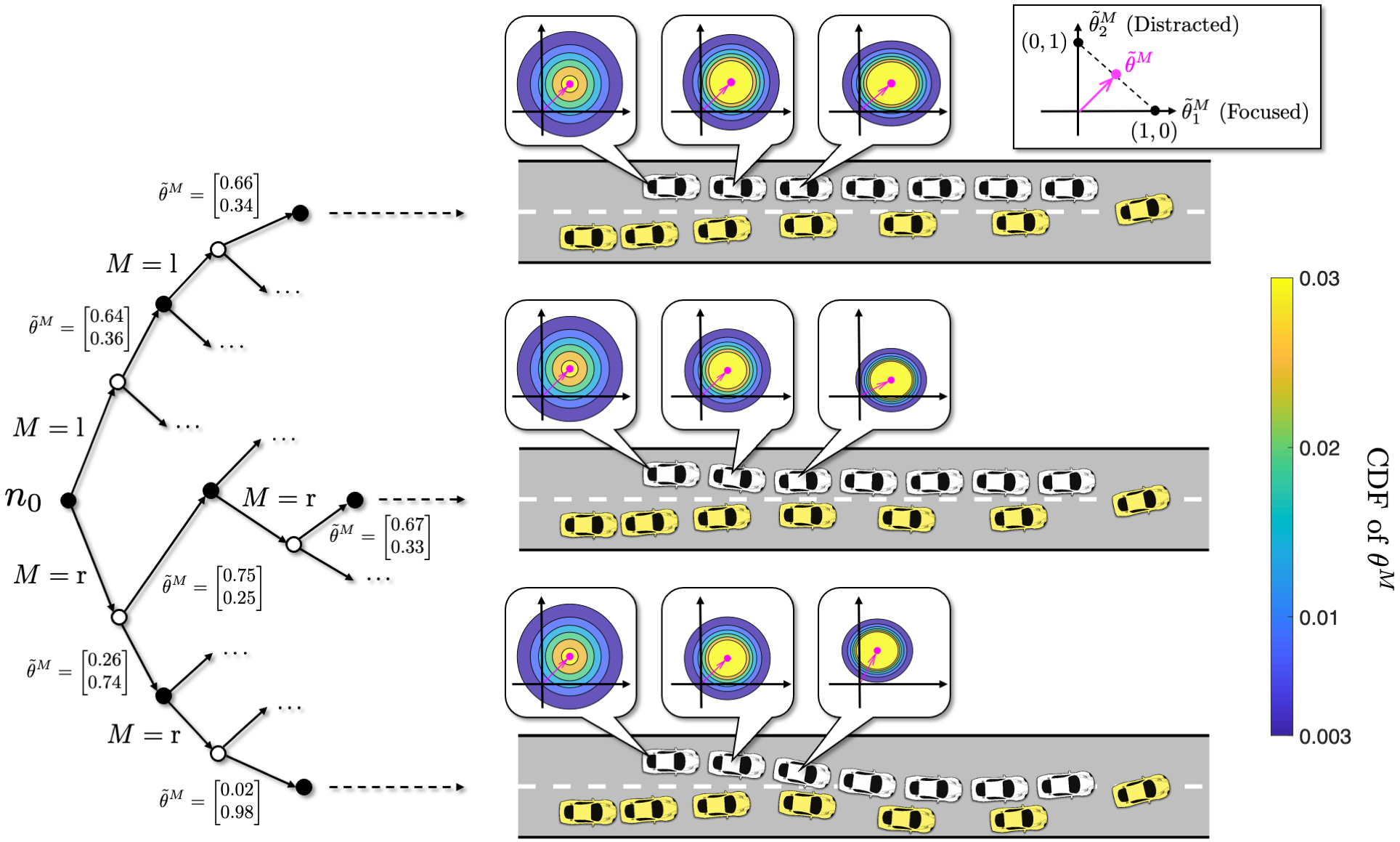}
  \caption{\label{fig:stree} Illustration of a scenario tree with
  $N^d = 2$ dual control time steps and $N^e = 4$ exploitation steps.
  The human-driven vehicle and autonomous car are plotted in white and yellow, respectively.
  The hidden state $\theta^M$ is modeled as a 2D Gaussian random variable, which is then projected onto a $1$-simplex by a softmax operation.
  The contour plots display level sets of the cumulative distribution function (CDF) of~$\theta^{M}$.
  Uncertainty is \emph{less} significantly reduced in the case where the human prefers the left lane (upper branch), since their behavior is less influenced by the robot's (probing) actions.
  \vspace{-3mm}
  }
\end{figure}
\noindent \textbf{Human Action Parameterization.}
In this paper, we parametrize the human's action at each time $t$ as a stochastic policy:
\begin{equation}
\label{eq:human_ctrl_model}
\textstyle\ctrl^H_t := \sum_{i=1}^{n_\theta} \theta_i^M  \ctrl_i^M, \quad \subjectto \ctrl^H_t \in \cset^H,
\end{equation}
which is a linear combination of \emph{stochastic basis policies} $\ctrl^M_i$ with parameter $\theta^M := (\theta^M_1,\theta^M_2,\ldots,\theta^M_{n_\theta}) \in \reals^{n_\theta}$.
We further allow each basis policy $\ctrl^M_i$ to have different \emph{modes} $M$ that take values from a \emph{finite} set $\mset$.
We define each basis policy $\ctrl^M_i$ with the ``noisily rational'' Boltzmann model from cognitive science~\cite{luce1959individual}.
Under this model, the human picks actions according to a probability distribution:
\begin{equation}
\label{eq:Boltzmann}
    \ctrl^M_i \sim 
    \distr\left(\ctrl^M_i \mid \state, \ctrl^R; M\right)=
    \frac{e^{-\qfunc_{i}^{M}\left(\state, \ctrl^R, \ctrl^M_i\right)}}{\int_{\tilde{\ctrl}^M_i \in {\cset}^{H}} e^{- \qfunc_{i}^{M}\left(\state, \ctrl^R, \tilde{\ctrl}^M_i\right)} d\tilde{\ctrl}^M_i },
\end{equation}
for all $i=1,2,\ldots,n_\theta$ and $M \in \mset$.
Here, $\qfunc_{i}^M(\state, \ctrl^R, \ctrl^M_i)$ is the human's state-action value (or Q-value) function associated with the $i$-th basis function and mode $M$.
This model assumes that, for a pair of fixed $(i,M)$, the human is exponentially likelier to pick an action that maximizes the state-action value function.

\begin{remark}
Our approach is agnostic to the concrete methods for determining the human's state-action value function $\qfunc_{i}^M(\state, \ctrl^R, \ctrl^M_i)$, model parameter $\theta^M$ and mode $M$, which are usually specified by the system designer based on domain knowledge or learned from prior data.
Goal-driven models of human motion are well-established in the literature. See for example~\cite{sadigh2018planning,ziebart2008maximum}.
We provide two examples below. 
\end{remark}

\begin{example}
\label{example:1}
Consider the highway driving scenario depicted in Fig.~\ref{fig:stree}, involving an autonomous car (R)
and a human-driven vehicle (H), whose action is parameterized as
\begin{equation*}
    u^H_t = \theta^M_D \policy^M_D \left(\state_t, u^R_t\right) + \theta^M_F \policy^M_F \left(\state_t, u^R_t\right), \quad \subjectto \ctrl^H_t \in \cset^H,
\end{equation*}
where $\theta^M_D$ and $\theta^M_F$ capture the level of distraction and focus of the human, respectively.
A focused human accounts for the safety of the joint system (e.g. avoiding and making room for the robot when it attempts to merge in front of the human), while a distracted human does not.
Modeling the human's level of distraction as a \emph{continuum} is motivated by recent work~\cite{toghi2021cooperative}, which differs from ours in that they assume an additive structure in the human's reward (objective) function rather than their policy.
Nonetheless, per the maxim of ``all models are wrong, but some are useful''~\cite{box1979all}, we will show in Section~\ref{sec:sim} that our choice of human predictive model is useful for planning purposes and leads to a policy that outperforms the state-of-the-art baseline.
Further, the human has two distinct modes, namely preferring to drive in the left lane or in the right lane, i.e., $M \in \{\text{l}, \text{r}\}$. 
Given a joint state $\state_t$, for each $i$ and $M$ we compute the human's game-theoretic state-action value function using the iterative linear-quadratic game method~\cite{fridovich2020efficient}, which yields an approximate local feedback Nash equilibrium solution.
\end{example}

\begin{example}
\label{example:2}
In the second example, we use parameter $\theta^M := (\theta^M_C, \theta^M_{NC})$ , which captures the human's level of cooperativeness.
A cooperative (C) human also optimizes for the robot's objective while a non-cooperative (NC) human does not.
We define the discrete modes as different interactive behaviors of the human toward the robot, similar to~\cite{bandyopadhyay2013intention,tian2021safety}.
Specifically, the human's state-action value function is defined as
\begin{equation*}
    \qfunc^M_i= \left\{\begin{array}{l}
    \begin{aligned}
    &\qfunc^M_i\left(\state_t, \ctrl^{R,\text{Nash}}_t(\state_t), \ctrl^M_i\right), \quad && \text{if } M = \textnormal{N (Nash)},\\
    &\qfunc^M_i\left(\state_t, \ctrl^{R,\text{worst}}_t(\state_t), \ctrl^M_i\right), \quad && \text{if } M = \textnormal{p (protected)},\\
    &\qfunc^M_i\left(\state_t, \ctrl^{R,\text{best}}_t(\state_t), \ctrl^M_i\right), \quad && \text{if } M = \textnormal{w (wishful)},\\
    &\qfunc^M_i\left(\state^H_t, \ctrl^M_i\right), \quad && \text{if } M = \textnormal{o (oblivious)}.
    \end{aligned}
    \end{array}\right.
\end{equation*}
\end{example}
In the first three modes, the human assumes that the robot's control $\ctrl^R_t$ is a local feedback Nash equilibrium solution~\cite{fridovich2020efficient}, the worst-case one that minimizes $\qfunc^M_i$, and the best-case one that maximizes $\qfunc^M_i$, respectively; the last mode follows the same assumption as in~\cite{fisacBHFWTD18,hu2022sharp} that the human ignores the presence of other agents.

\noindent \textbf{Inferring Model Parameter and Mode.}
In general, parameter $\theta^M$ and mode $M$ in human action model \eqref{eq:human_ctrl_model} are \emph{hidden states} that are unknown to the robot.
Therefore, they can only be inferred from past observations.
To address this, we define the information vector $\ivec_t := \left[\state_{t}, \ctrl^R_{t-1}, \ivec_{t-1}\right]$ as the collection of all information that is \textit{causally observable} by the robot at time $t \geq 0$, with $\ivec_{0}=\left[\state_{0}\right]$.
We then define the \textit{belief state} ${\bel_t := \distr\left(\theta^M, M \mid \ivec_{t}\right)}$ as the joint distribution of $(\theta^M,M)$ conditioned on $\ivec_t$, and $\bel_0:=\distr\left(\theta^M, M \right)$ is a given prior distribution.
When the robot receives a new observation $\state_{t+1}\in \ivec_{t+1}$, the current belief state $\bel_t$ is updated using the recursive Bayesian inference:
\begin{subequations}
\begin{align}
&\distr(\theta^M_{-} \mid \ivec_{t+1}; M) = \frac{ \distr(\state_{t+1} \mid \ctrl^R_t, \ivec_t; \theta^M, M) \distr(\theta^M \mid \ivec_{t}; M) }{ \distr(\state_{t+1} \mid \ctrl^R_t, \ivec_t; M)} \label{eq:Bayes_est_meas_theta}, \\
&\distr(M_{-} \mid \ivec_{t+1}) = \frac{\distr(\state_{t+1} \mid \ctrl^R_t, \ivec_t; M) \distr(M \mid \ivec_{t}) }{ \distr(\state_{t+1} \mid \ctrl^R_t, \ivec_t)} \label{eq:Bayes_est_meas_M}, \\
&\bel^-_{t+1} := 
\distr(\theta^{M}_{-}, M_{-} \mid \ivec_{t+1}) = \distr(\theta^{M}_{-} \mid \ivec_{t+1}; M) \distr(M_{-} \mid \ivec_{t+1}), \\
&\bel_{t+1} = g^t(\bel^-_{t+1}) := \textstyle \int  \distr(\theta^M, M \mid \tilde{\theta}^M_{-}, \tilde{M}_{-}) \distr(\tilde{\theta}^M_{-}, \tilde{M}_{-} \mid \ivec_{t+1}) d \tilde{\theta}^M_{-} d \tilde{M}_{-}. \label{eq:Bayes_est_time} 
\end{align}
\end{subequations}
where $\distr(\theta^M, M \mid \tilde{\theta}^M_{-}, \tilde{M}_{-})$ is a transition model and $g^t(\cdot)$ is the belief state transition dynamics.
Compactly, we can rewrite~\eqref{eq:Bayes_est_meas_theta}-\eqref{eq:Bayes_est_time} as a dynamical system,
\begin{equation}
\label{eq:belief_state_dyn}
{\bel}_{t+1}=g({\bel}_{t}, \state_{t+1}, \ctrl^R_{t} ).
\end{equation}
Unfortunately, system \eqref{eq:belief_state_dyn} in general does not adopt an analytical form beyond one-step evolution.
Even if the prior distribution $\bel_t$ is a Gaussian, the posterior $\bel_{t+1}$ ceases to be a Gaussian since the human's action $\ctrl^H_t$ defined by \eqref{eq:human_ctrl_model} and \eqref{eq:Boltzmann} (which in turn affects the observation $\state_{t+1}$) is generally non-Gaussian, thus precluding the use of the conjugate-prior properties of Gaussian distributions.
In Section~\ref{sec:main}, we will introduce a computationally efficient method to propagate the belief state dynamics approximately.

\section{Problem Statement}

\noindent \textbf{Canonical Human-in-the-Loop Planning Problem.}
We now define the central problem we want to solve in this paper: the canonical human-in-the-loop planning problem, which is formulated as a stochastic finite-horizon optimal control problem as follows:
\begin{subequations}
\label{eq:HRI}
\begin{align}
\label{eq:HRI:obj} \min_{\substack{\policy^R_{[0:N-1]}}} \
&\expectation\limits_{\substack{ (\theta^M, M) \sim \bel_{[0:N-1]}, \\ \ctrl_{[0: N-1]}^{H}\sim \eqref{eq:human_ctrl_model}, \dstb_{[0:N-1]} } } 
\sum_{k=0}^{N-1} \ell^R \left(\state_k, \policy^R_k(\state_k,\bel_k)\right) + \ell_F^R (\state_N) \\
\text{s.t.} \quad \label{eq:HRI:sys_init} &\state_{0}=\hat{\state}_{t}, \ \bel_{0}=\hat{\bel}_{t}, \\
\label{eq:HRI:sys_dyn} &\state_{k+1} = \dyn (\state_k) + B^R \policy^R_k(\state_k, \bel_k) + B^H \ctrl^H_k + \dstb_k, &&\hspace{-1.0cm} \forall k = 0,\ldots,N-1 \\
\label{eq:HRI:belief_dyn} &{\bel}_{k+1}=g\left({\bel}_{k}, \state_{k+1}, \policy^R_{k}(\state_k, \bel_k)\right),  &&\hspace{-1.0cm} \forall k = 0,\ldots,N-1 \\
\label{eq:HRI:safety} &\state_k \notin \failure,  &&\hspace{-1.0cm} \forall k = 0,\ldots,N
\end{align}
\end{subequations}
where $\hat{\state}_{t}$ and $\hat{\bel}_{t}$ are the state measured and belief state maintained at time $t$, $\ell^R: \reals^{n} \times \cset^R \rightarrow \reals_{\geq 0}$ and  $\ell^R_F: \reals^{n} \rightarrow \reals_{\geq 0}$ are designer-specified stage and terminal cost function, $\policy_k(\state_k, \bel_k)$ is a \emph{causal} feedback policy~\cite{bar1974dual,mesbah2018stochastic} that leverages the (yet-to-be-acquired) knowledge of future state $\state_k$ and belief state $\bel_k$, and $\failure \subseteq \reals^n$ is a failure set that the state is not allowed to enter.

In theory, problem \eqref{eq:HRI} can be solved using stochastic dynamic programming~\cite{bellman1966dynamic}.
An optimal robot's value function $\valfunc_k(\state_k, \bel_k)$ and control policy $\policy_k^{R,*}(\state_k, \bel_k)$ can be obtained backwards in time using the Bellman recursion,
\begin{equation}
\label{eq:sto_DP}
\begin{aligned}
\valfunc_k(\state_k, \bel_k) = &\min_{\substack{
\policy_k(\state_k,\bel_k)}} \ell^R(\state_k, \ctrl^R_k) + \expectation\limits_{\substack{ (\theta^M, M) \sim \bel_k, \\ \ctrl_k^{H}\sim \eqref{eq:human_ctrl_model}, \dstb_k } } \left[\valfunc_{k+1}(\state_{k+1}, \bel_{k+1}) \mid \ivec_k \right] \\
& \ \quad \text{s.t.} \ \ \eqref{eq:HRI:sys_dyn}-\eqref{eq:HRI:safety}
\end{aligned}
\end{equation}
with terminal condition $\valfunc_N(\state_N,\bel_N) = \ell^R_F(\state_N)$.

\noindent \textbf{Dual Control Effect.}
Value function $\valfunc_t(\state_t, \bel_t)$ obtained by solving \eqref{eq:sto_DP} depends on future belief states $\bel_k\; (k>t)$, thus giving the optimal policy $\ctrl^{R,*}_t :=  \policy^R_0(\hat{\state}_t, \hat{\bel}_t)$ the ability to affect \emph{future uncertainty} of the human quantified by the belief states.
Therefore, the optimal policy $\ctrl^{R,*}_t$ of \eqref{eq:sto_DP} possesses the property of \emph{dual control effect}, defined formally in Definition~\ref{def:DC}.
Due to the principle of optimality~\cite{bellman1966dynamic}, the policy achieves an optimal balance between optimizing the robot's expected performance objective~\eqref{eq:HRI:obj} and actively reducing \new{its uncertainty about the human}.
In other words, the optimal policy of~\eqref{eq:sto_DP} automatically probes the human agents to reduce their associated uncertainty \emph{only} to the extent that doing so improves the robot's \new{expected} closed-loop performance.

\begin{definition}[Dual Control Effect~\cite{feldbaum1960dual,bar1974dual,mesbah2018stochastic,hijab1984entropy}]
\label{def:DC}
A control input has dual control effect if it can affect, with nonzero probability, (a) at least one $r$th-order ($r \geq 2$) central moment of a hidden state variable in a metric space,
or (b) the entropy of a categorical hidden state variable.
\end{definition}

\noindent \textbf{Complication and Approximation of Dual Control.}
Unfortunately, \eqref{eq:sto_DP} is computationally intractable in all but the simplest cases, mainly due to nested optimization of robot's action and and computing the conditional expectation.
The expectation term in \eqref{eq:sto_DP} can be approximated to arbitrary accuracy with quantization of the belief states, which, however, leads to exponential growth in computation, i.e. the issue of \emph{curse of dimensionality}~\cite{bellman1966dynamic}.
It is for those reasons that approximate methods are mainly used to solve dual control problems.
Approximate dual control can be categorized into: explicit approaches, e.g.~\cite{heirung2015mpc,sadigh2018planning,tian2021anytime} that simplify the original stochastic optimal control problem by artificially introducing probing effect or information gathering objectives to the control policy, and implicit approaches, e.g.~\cite{bar1974dual,klenske2016dual,arcari2020dual} that rely on direct approximation of the Bellman recursion~\eqref{eq:sto_DP}.
The approach we take in this paper is a \emph{scenario-based implicit dual control method}, which is detailed in the next section.
The main advantage of using the implicit dual control approximation is that the automatic exploration-exploitation trade-off of the policy is naturally preserved in an optimal sense~\cite{sehr2017tractable,mesbah2018stochastic}.

\section{Active Uncertainty Reduction for Human-Robot Interaction}
\label{sec:main}
In this section, we describe an implicit dual control approach towards approximately solving the canonical human-in-the-loop planning problem~\eqref{eq:HRI}.
We start by presenting two approximation schemes for propagating the belief state dynamics and computing the expectation in Bellman recursion~\eqref{eq:sto_DP}, which are necessary for reformulating~\eqref{eq:sto_DP} as a real-time solvable SMPC problem.
The formulation and properties of our proposed SMPC problem are shown towards the end of this section.
See Fig.~\ref{fig:stree} for an example solution of our proposed SMPC problem.

\noindent \textbf{Tractable Reformulation of Belief State Dynamics.}
We start by deriving a tractable recursive update rule for the belief state dynamics \eqref{eq:belief_state_dyn} by approximating the human's action model \eqref{eq:human_ctrl_model}, for a given $(M,\theta^M)$, as a Gaussian distribution.
The technical tool we rely on is the Laplace approximation~\cite{DraganS12}.
Precisely, the conditional probability distribution of each human's basis policy $\ctrl^M_i$ defined in \eqref{eq:Boltzmann} is approximated as:
\begin{equation}
\label{eq:Laplace}
    \ctrl^M_i \sim 
    \distr\left(\ctrl^M_i \mid \state_t, \ctrl^R_t; M\right)
    \approx \gaussian\left(\mu_i^M(\state_t, \ctrl^R_t), \covar_i^M(\state_t, \ctrl^R_t) \right),
\end{equation}
where the mean function is 
$\mu_i^M(\state_t, \ctrl^R_t) = \textstyle\argmax_{\ctrl_{i}^{M} \in \cset^H} \qfunc_{i}^{M}\left(\state_{t}, \ctrl^R_t, \ctrl_{i}^{M}\right)$
and the covariance function is 
$\covar_i^M(\state_t, \ctrl^R_t) = \textstyle \nabla^2_{\ctrl_i^M} \left. {\qfunc_{i}^{M}\left(\state_{t}, \ctrl^R_t, \ctrl_{i}^{M}\right)}^{-1}\right|_{\ctrl_i^M = \mu_i^M(\state_t, \ctrl^R_t)}.$
The intuition behind the above Laplace approximation scheme is that the Gaussian distribution obtained in \eqref{eq:Laplace} centers around the mode of the original basis policy distribution $\distr\left(\ctrl_{i}^{M} \mid \state_t, \ctrl^R_t,M\right)$, i.e. $\mu_i^M(\state_t, \ctrl^R_t)$, which corresponds to the \emph{perfectly rational} human action associated with $\theta_i^M$.
Hence, the overall human action distribution, conditioned on $\theta^M$ and $M$, is given by:
\begin{equation}
\label{eq:human_control_model_approx}
\textstyle\ctrl^M_t \sim  \gaussian\left(\sum_{i=1}^{n_\theta} \theta_i^M \mu_i^M(\state_t, \ctrl^R_t), \sum_{i=1}^{n_\theta} \left(\theta_i^M\right)^2 \covar_i^M(\state_t, \ctrl^R_t) \right).
\end{equation}
We \new{subsequently lift} the requirement
that $\ctrl^H_t \in \cset^H$ in order to keep $\ctrl^H_t$ normally distributed \new{\emph{during belief propagation}, and we use projected $\ctrl^H_t$ for state evolution.}
Now, we can rewrite the human-robot joint system \eqref{eq:joint_sys} as
\begin{equation}
\label{eq:joint_sys_approx}
    \state_{t+1} = B^H(\state_t) U^M(\state_t, \ctrl^R_t) \theta^M + \dyn (\state_t) + B^R(\state_t) \ctrl^R_t + \bar{\dstb}^M_t,
\end{equation}
where $U^M(\state_t, \ctrl^R_t) := \begin{bmatrix} \mu_1^M(\state_t, \ctrl^R_t) &\mu_2^M(\state_t, \ctrl^R_t) &\ldots &\mu_{n_\theta}^M(\state_t, \ctrl^R_t) \end{bmatrix}$ is a deterministic state- and input-dependent matrix,
and the combined disturbance term $\bar{\dstb}^M_t$ is a zero-mean Gaussian, whose covariance is given by
\begin{equation}
\label{eq:combined_dstb_covar}
\begin{aligned}
\covar^{\bar{\dstb}^M_t} (\state_t, \ctrl^R_t; \theta^M) 
:= \covar_\dstb + \textstyle\sum_{i=1}^{n_\theta} \left(\theta_i^M\right)^2 B^H(\state_t) \covar_i^M(\state_t, \ctrl^R_t) {B^H(\state_t)}^\top.
\end{aligned}
\end{equation}
Note that even if \eqref{eq:joint_sys_approx} is linear in $\theta^M$, dependence of covariance matrix $\covar^{\bar{\dstb}^M_t}$ on $\theta^M$, as shown in \eqref{eq:combined_dstb_covar}, still prohibits updating the belief states in closed-form.
To this end, we approximate covariance $\covar^{\bar{\dstb}^M_t}$ by fixing $\theta^M$ with some estimated value $\bar{\theta}^M$, which can be obtained, for example, from solutions of the last run or roll-out-based simulations.

Now, we can readily propagate approximate belief state dynamics \eqref{eq:belief_state_dyn} efficiently in closed-form, leveraging the conjugate prior property of Gaussian distributions.
Crucially, the above approximate belief state updates preserve the dual control effect.
To verify this claim, we examine the covariance of the updated conditional distribution $\distr(\theta^M_{-} \mid \ivec_{t+1}; M)$:
\begin{equation}
\label{eq:updated_covar}
    \covar^{\theta^M_{-}}_{t+1} =  \covar^{\theta^M}_{t} + \left[ \left( B^H U^M(\state_t, \ctrl^R_t) \right)^\top {\covar^{\bar{\dstb}^M_t}}^{-1} (\state_t, \ctrl^R_t) B^H U^M(\state_t, \ctrl^R_t)\right]^{-1},
\end{equation}
whose prior distribution is $\distr(\theta^M \mid \ivec_{t}; M) = \gaussian(\mean^{\theta^M}_{t}, \covar^{\theta^M}_{t})$. 
From the above we can clearly see that the robot's control $\ctrl^R_t$ affects the updated covariance matrix $\covar^{\theta^M_{-}}_{t+1}$, and all future covariance matrices implicitly by affecting future states $\state_k$ ($k>t+1$) via the Bellman recursion, hence producing dual control effect for $\theta^M$ according to Definition~\ref{def:DC}.
Dual control effect for $M$ can be similarly verified via~\eqref{eq:Bayes_est_meas_M}.
We denote the approximate belief state dynamics as ${\bel}_{t+1}=\tilde{g}({\bel}_{t}, \state_{t+1}, \ctrl^R_{t} )$.


\noindent \textbf{Implicit Dual Control using Scenario Trees.}
In this section, we propose an approximate solution method for the canonical human-in-the-loop planning problem \eqref{eq:HRI} using scenario-tree-based stochastic model predictive control (ST-SMPC)~\cite{bernardini2011stabilizing,mesbah2018stochastic}, which yields a control policy with dual control effect.
The key idea of ST-SMPC is to approximate the expectation in Bellman recursion \eqref{eq:sto_DP} based on uncertainty samples, i.e. quantized belief states.
This leads to a scenario tree that allows us to roll out~\eqref{eq:sto_DP} as a deterministic finite-horizon optimal control problem, which can be readily solved by gradient-based algorithms.
Since our approach hinges on directly approximating Bellman recursion~\eqref{eq:sto_DP}, it can be understood as an implicit dual control method~\cite{mesbah2018stochastic}.

\remove{A typical scenario tree is illustrated in Fig.~\ref{fig:stree}.
The scenario tree construction procedure is summarized in Alg.~\ref{alg:tree} in Appendix~\ref{apdx:tree}.}
We denote a \emph{node} in the scenario tree as $\node$, whose time, state and belief state are denoted as $t_n$, $\state_n$ and $\bel_\node$, respectively.
Similarly, the uncertainty samples of the node are $\theta^M_n$, $\bar{\dstb}^M_n$ and $M_n$.
\new{Here, recall that $\bar{\dstb}^M_n$ is the combined disturbance defined in~\eqref{eq:combined_dstb_covar}, which implicitly incorporates a sample of the human's action.}
The set of all nodes is defined as $\nodeset$.
We define the transition probability from a parent node $\pre{\node}$ to its child node $\node$ as $\bar{\prob}_{\node} := \prob(\theta^M_n \mid \ivec_{\pre{\node}}; M_\node) \prob(\bar{\dstb}^{M}_n \mid \ivec_{\pre{\node}}; M_\node) \prob\left(M_\node \mid \ivec_{\pre{\node}}\right).$
Subsequently, the \textit{path transition probability} of node $\node$, i.e. the transition probability from the root node $\node_0$ to node $\node$ can be computed recursively as $\prob_{\node} := \bar{\prob}_{\node} \cdot \bar{\prob}_{\pre{\node}} \cdots \bar{\prob}_{\node_0}$.
In order to quickly compute the conditional probabilities of $(\theta^M_n, \bar{\dstb}^{M}_n)$ and avoid online sampling (i.e. during the optimization), we use an offline sampling procedure, leveraging the fact that they are Gaussian random variables, similar to what is done in~\cite{bonzanini2020safe}.
We first generate samples offline from the standard Gaussian distribution. \remove{, i.e. $\theta^{M,o}_n \sim \gaussian(0,I)$ and $\bar{\dstb}^{M,o}_n \sim \gaussian(0,I)$.}
Then, during online optimization, these samples are transformed using the analytical mean and covariance:
\begin{subequations}
\begin{align}
    \label{eq:sample_trans_theta}
    \theta^M_n &= \mean^{\theta^M_{\node}}(\state_\node, \ctrl^R_\node) + \left(\covar^{\theta^M_{\node}}(\state_\node, \ctrl^R_\node)\right)^{1/2} \theta^{M,o}_n, \\
    \label{eq:sample_trans_dstb}
    \bar{\dstb}^{M}_n &=  \left(\tilde{\covar}^{\bar{\dstb}^M_\node} (\state_\node, \ctrl^R_\node)\right)^{1/2} \bar{\dstb}^{M,o}_n.
\end{align}
\end{subequations}
Unlike existing ST-SMPC methods for human-robot interaction, such as~\cite{schildbach2015scenario,hu2022sharp},
our scenario tree has both \emph{state- and input-dependent} uncertainty realizations (via transformations~\eqref{eq:sample_trans_theta} and~\eqref{eq:sample_trans_dstb}) and path transition probabilities (via Bayesian update of $M$ in~\eqref{eq:Bayes_est_meas_M}).
Therefore, the nodes can move in response to predicted states and inputs during online optimization.
Fundamentally, it is this feature that allows the robot to actively reduce the human-related uncertainty via the dual control effect, and capture mutual responses between the human and the robot.
\begin{remark}
\new{In order to alleviate the exponential growth of complexity associated with the scenario tree, only a small set of $M$, $\theta^{M,o}_n$ and $\bar{\dstb}^{M,o}_n$ are sampled.
A scenario pruning mechanism is introduced in~\cite{hu2022sharp} for static scenario trees, which can be used here as a heuristic to prune branches based on initial guesses.
Principled pruning for scenario trees involving state- and input-dependent samples, however, remains an open question.}
\end{remark}

\noindent \textbf{Exploitation Steps.}
In order to alleviate the computation challenge caused by the exponential growth of nodes in the scenario tree, we can stop branching the tree at a stage $N^d < N$, which we refer to as the dual control horizon. 
Subsequently, the remaining $N^e := N - N^d$ stages become the exploitation horizon, where each scenario is extended without branching
and the belief states are only propagated with the transition dynamics $g^t(\cdot)$ defined in~\eqref{eq:Bayes_est_time}, corresponding to a non-dual SMPC problem.
Thanks to the scenario tree structure, control inputs of the exploitation steps still preserve the causal feedback property, allowing the robot to be cautious and ``passively adaptive'' to future human's uncertainty realizations.

\noindent \textbf{Overall ST-SMPC Problem.}
\label{sec:SMPC}
Given a scenario tree defined by node sets $\nodeset_t$, we can roll out the Bellman recursion \eqref{eq:sto_DP} based on uncertainty samples in the tree, leading to a ST-SMPC problem formulated as
\begin{equation}
\label{eq:ST-SMPC}
\begin{aligned}
\min_{\substack{\mathbf{U}^R_t}} \ \ 
&\sum_{\tnode \in \nodeset_t \setminus \leafset_t} \prob_{\tnode} \ell^R (\state_\tnode, \ctrl^R_\tnode) + \sum_{\tnode \in \leafset_t} \prob_{\tnode} \ell^R_{F} (\state_\tnode, \bel_\tnode) \\
\text{s.t.} \ \  &\state_{\node_0}=\hat{\state}_{t}, \ \bel_{\node_0}=\hat{\bel}_{t}, \\
&\state_{\tnode} =  \dyn (\state_{\pre{\tnode}}) + B^R(\state_{\pre{\tnode}}) \ctrl^R_{\pre{\tnode}} + B^H(\state_{\pre{\tnode}}) \ctrl^H_{\tnode} + \bar{\dstb}^{M}_{\tnode}, &&\hspace{0.25cm} \forall \tnode \in \nodeset_t \setminus \{n_0\} \\
&{\bel}_{\tnode}=\tilde{g}\left({\bel}_{\pre{\tnode}}, \state_{\tnode}, \ctrl^R_{\pre{\tnode}}\right),  &&\hspace{0.25cm} \forall \tnode \in \nodeset_t^d \\
&{\bel}_{\tnode}=g^t\left({\bel}_{\pre{\tnode}}\right), &&\hspace{0.25cm} \forall \tnode \in \nodeset_t^e \\
&\textstyle\ctrl^H_{{\tnode}}=U^{M_{{\tnode}}}(\state_{\pre{\tnode}}, \ctrl^R_{\pre{\tnode}}) \theta^{M}_{{\tnode}},\;\ctrl^H_{{\tnode}} \in \cset^H,\;\eqref{eq:sample_trans_theta},\;\eqref{eq:sample_trans_dstb}, &&\hspace{0.25cm} \forall \tnode \in \nodeset_t \setminus \{n_0\} \\
&\ctrl^R_\tnode \in \cset^{R}, &&\hspace{0.25cm} \forall \tnode \in \nodeset_t \setminus \leafset_t\\
&\state_{\tilde{n}} \notin \failure, &&\hspace{0.25cm} \forall \tnode \in \nodeset_t\\
\end{aligned}
\end{equation}
where $\leafset_t$ is the set of all leaf nodes,
(i.e. ones that do not have a descendant),
$\nodeset^d_t$ and $\nodeset^e_t$ are the set of dual control and exploitation nodes, respectively, and set $\mathbf{U}^R_t := \{ \ctrl^R_\tnode \in \reals^{m_R}: \tnode \in \nodeset_t \setminus \leafset_t \}$ is the collection of robot's control inputs associated with all non-leaf nodes.
\new{Problem~\eqref{eq:ST-SMPC} is a nonconvex trajectory optimization problem, which can be solved using general-purpose nonconvex solvers such as SNOPT~\cite{gill2005snopt}.}
The optimal solution $\mathbf{U}^{R,*}_t$ to \eqref{eq:ST-SMPC} is implemented in a receding horizon fashion, i.e. $\pi^R_{\text{ID-SMPC}}(\hat{\state}_t, \hat{\bel}_t) := \ctrl_{\node_0}^{R,*}$.

\begin{theorem}
The policy $\pi^R_{\text{ID-SMPC}}(\cdot,\cdot)$ obtained by solving~\eqref{eq:ST-SMPC} has dual control effect.
\end{theorem}
\begin{proof}
From~\eqref{eq:updated_covar}, the robot's control $\ctrl^R_t$ can affect the covariance of the belief over~$\theta^M$.
Therefore, the policy $\pi^R_{\text{ID-SMPC}}$ has dual control effect for $\theta^M$ per Definition~\ref{def:DC}.
Similarly, from~\eqref{eq:Bayes_est_meas_M}, it can be seen that $\ctrl^R_t$ can affect all components of the categorical distribution over $M$, and thereby its entropy, implying dual control effect for $M$ per Definition~\ref{def:DC}.
\end{proof}

\noindent \textbf{Safety and Feasibility.}
Since the ST-SMPC problem~\eqref{eq:ST-SMPC} is formulated as a standard scenario program, safety assurances can be obtained based on existing safe control methods.
In the following, we briefly discuss three ways to enforce safety, i.e. the human-robot joint system in closed-loop with the policy of~\eqref{eq:ST-SMPC} satisfying $\state \notin \failure$.
\begin{figure}[!hbtp]
     \centering
     \includegraphics[width=1.0\columnwidth]{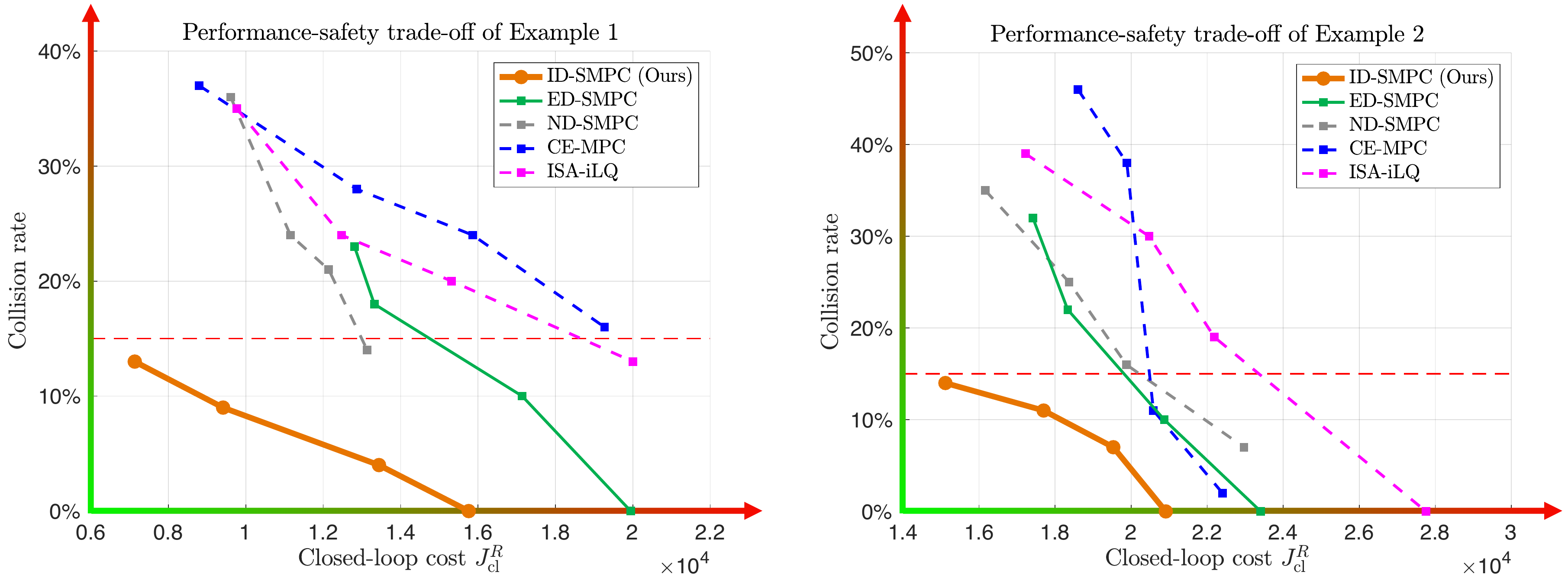}
     \caption{\label{fig:Ex1_tradeoff} Performance and safety trade-off of Examples 1 and 2. 
      Each data point is obtained based on a particular design of the planner and the statistical data of 100 trials with different random seeds. The collision rate of ID-SMPC is always below $15\%$ even for the least conservative design.
     }
\end{figure}
\begin{itemize}
    \item \emph{Robust.} Our prior work~\cite{hu2022sharp} provides robust safety and recursive feasibility guarantees for ST-SMPC based on shielding, a “least-restrictive” supervisory safety filter designed purely based on sets $\failure$, $\cset^H$, $\cset^R$, and is agnostic to the specific predictive model of human motion.
    As a special case,
    \cite{bonzanini2020safe} provides closed-loop safety guarantees by projecting the ST-SMPC policy onto the set of backup control policies associated with a robust controlled-invariant set.
    \new{Note that an additional assumption of bounded disturbance is needed in order to obtain robust safety guarantees.}
    
    \item \emph{Probabilistic.} We can replace~\eqref{eq:HRI:safety} with \emph{chance constraints} to obtain probabilistic safety guarantees, i.e. $\prob\left[\state \in \failure\right] \leq \beta$, where $\beta$ is the tolerance level, provided that $\state \in \failure$ can be written as a set of inequality constraints.
    An analytical bound of $\beta$ and probabilistic feasibility guarantee are established for chance-constrained ST-SMPC problems with dynamic obstacles in~\cite{de2021scenario}.
    
    \item \emph{Soft constrained.} Soft constraint is a simple, yet effective technique widely used by MPC to enforce safety of the closed-loop system.
    The advantages of soft constraints are that they are easy to implement, efficient to optimize, and feasibility can be guaranteed.
    In this paper, we tailor the soft constrained MPC approach in~\cite{zeilinger2014soft} for ST-SMPC, which relaxes the original hard constraints $\state_\tnode \notin \failure$ in~\eqref{eq:ST-SMPC} with slack variables for each node $\tnode \in \nodeset_t$.
    Although satisfaction of $\state \notin \failure$ can no longer be guaranteed for the closed-loop system, we show in Section~\ref{sec:sim} that our approach significantly reduces the collision rate comparing to the baselines (using the same soft constrained MPC approach) due to active reduction of the human's uncertainty.
\end{itemize}

\section{Simulation Results}
\label{sec:sim}

We evaluate our proposed implicit dual scenario-tree-based SMPC (ID-SMPC) planner on simulated driving scenarios, where the human-driven vehicles are simulated using a game-theoretic model synthesized with~\cite{fridovich2020efficient}.
The open-source code is available online.\footnote{\url{https://github.com/SafeRoboticsLab/Dual_Control_HRI}}
\begin{figure}[!hbtp]
  \centering
  \includegraphics[width=0.91\columnwidth]{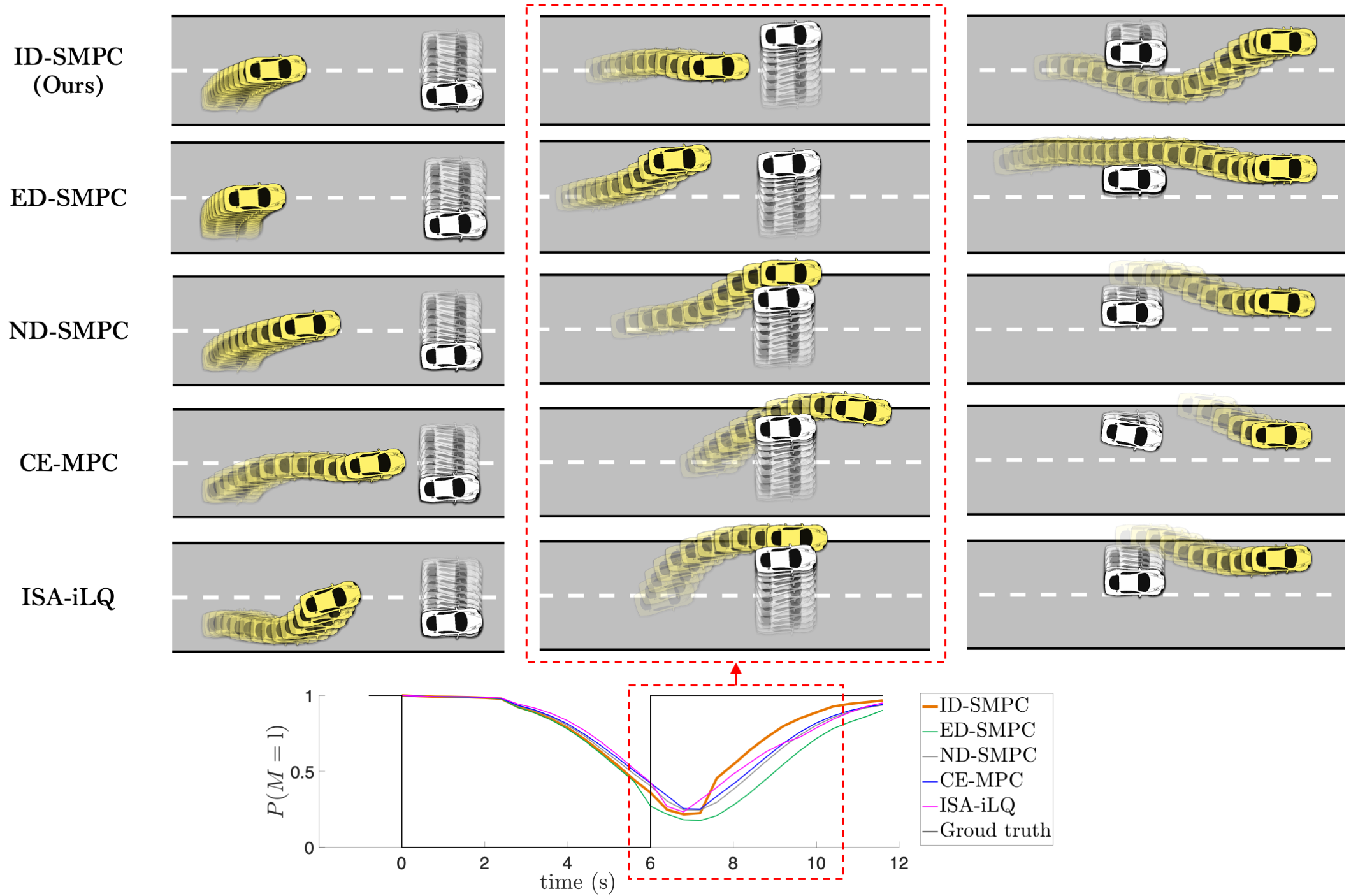}
  \caption{\label{fig:Ex1_Trajs} Simulation snapshots of Example 1. Longitudinal positions are shown in relative coordinates with $p_x^H = 0$. The left, middle, and right columns display trajectories for $t=[0,3]$ s, $t=[3,5]$ s, and the remainder of the trajectories, respectively.
  The bottom figure shows probability $\prob(M=\text{l})$ for all four planners overtime.
  Our proposed ID-SMPC planner yielded a clean and sharp overtaking maneuver of the robot while the non-dual planners led to unsafe trajectories.
  }
\end{figure}

\noindent \textbf{Baselines.}
We compare our proposed ID-SMPC planner against four baselines:
\begin{itemize}
    \item Explicit dual SMPC (ED-SMPC) planner, which augments the stage cost $\ell^R$ in \eqref{eq:HRI:obj} with an information gain term $\lambda (\entropy(\bel_{k})-\entropy(\bel_{k+1}))$ proposed by~\cite{sadigh2018planning}.
    Here, $\lambda > 0$ is a fine-tuned weighting factor.
    
    \item Non-dual scenario-tree-based SMPC (ND-SMPC) planner, which is based on solving~\eqref{eq:HRI} with a scenario tree that does not propagate belief states with an observation model (so the resulting policy does not have dual control effect).
    A similar scenario program is also considered in~\cite{bernardini2011stabilizing,schildbach2015scenario,hu2022sharp}.
    
    \item Certainty-equivalent MPC (CE-MPC) planner, which is based on solving~\eqref{eq:HRI} with the certainty-equivalence principle~\cite{mesbah2018stochastic,arcari2020dual}.
    
    \item Iterative linear-quadratic game with inference-based strategy alignment (ISA-iLQ) planner, proposed by~\cite{peters2020inference} and originally developed in~\cite{fridovich2020efficient}.
    An input projector is used to enforce $\ctrl^H \in \cset^H$ and $\ctrl^R \in \cset^R$.
\end{itemize}
ED-SMPC is a dual control planner while the other three do not generate dual control effect.
All planners use the same quadratic cost functions $\ell^R$ and $\ell_F^R$.

\noindent \textbf{Simulation Setup.} 
All planners are equipped with the same human intent inference scheme.
Vehicle and pedestrian dynamics are described by the kinematic bicycle model in~\cite{fisac2019hierarchical} and the unicycle model in~\cite{fridovich2020efficient}, respectively, both discretized with a time step of $\Delta t = 0.2$ s.
The interactive human agents are simulated using the iterative linear-quadratic game method~\cite{fridovich2020efficient} with a game horizon of $3$~s.
All simulations are performed using MATLAB and YALMIP~\cite{lofberg2004yalmip} on a desktop with an Intel Core i7-10700K CPU.
All nonlinear MPC problems are solved with SNOPT~\cite{gill2005snopt}.
\begin{figure}[!hbtp]
  \centering
  \includegraphics[width=1.0\columnwidth]{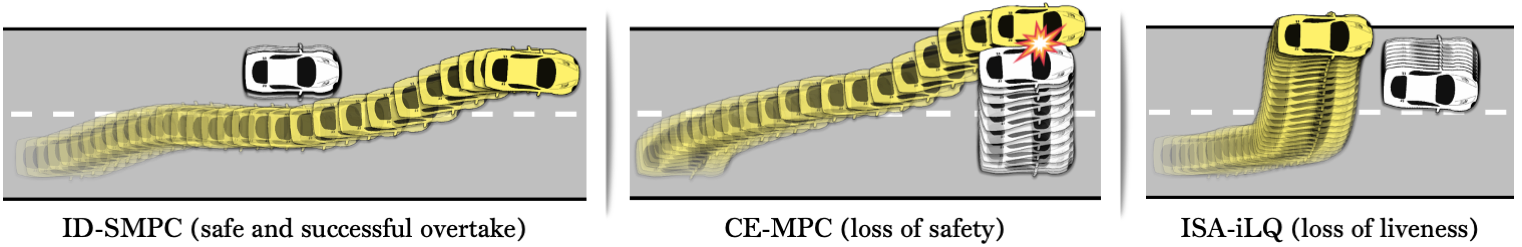}
  \caption{\label{fig:Ex1_safety_liveness} One trial of Example 1, where the robot successfully overtook the human-driven vehicle in 6 s using ID-SMPC.
  CE-MPC and ISA-iLQ led to loss of safety and liveness, respectively.
  }
\end{figure}
The average solving time for ISA-iLQ, CE-MPC, ND-SMPC, ED-SMPC, and ID-SMPC \new{at each time step} are
1.92~s, 0.333~s, 0.184~s, 0.481~s, and 0.478~s,
respectively, when the robot is interacting with one human agent (Section~\ref{sec:Ex1}~and~\ref{sec:Ex2}), and
19.3~s, 0.775~s, 1.37~s, 4.17~s, and 4.03~s,
respectively, in the presence of three human agents (Section~\ref{sec:Ex3}).
The increase in solving time (for ST-SMPC problems) is due to the exponential growth of scenarios when the number of agents increases.
This highlights an opportunity to
improve \remove{the }time efficiency \remove{of our algorithm }in multi-agent settings leveraging sparse scenario tree methods~\cite{hu2022sharp} and distributed MPC techniques.

\noindent \textbf{Metrics.}
To measure the planning performance, we consider the following two metrics:
\begin{itemize}
    \item Closed-loop cost, defined as $J^R_\cl := \sum_{t=0}^{T_\Sim} \ell^R(\state_t, \ctrl_t^R)$, where $T_\Sim$ is the simulation horizon, and $\state_{[0:T_\Sim]}, \ctrl_{[0:T_\Sim]}$ are the \emph{executed} trajectories (with replanning).
    \item Collision rate, defined as $N_{\text{coll}} / N_{\text{trial}} \times 100\%$, where $N_{\text{coll}}$ is the number of trials that a collision happens, i.e. $x_t \in \failure$,  and $N_{\text{trial}}$ is the total number of trials.
\end{itemize}

\noindent \textbf{Hypotheses.} 
We make three hypotheses, which are supported by our simulation.
\begin{itemize}
    \item \textbf{H1 (Performance and Safety Trade-off).} \emph{Dual-control planners result in a better performance-safety trade-off than non-dual baselines.}
    
    \item \textbf{H2 (Safety and Liveness).} \emph{Insufficient knowledge about human hidden states (due to lack of dual control effect) can cause loss of safety (e.g. collisions) and/or liveness (e.g. failing to overtake the human).}
    
    \item \textbf{H3 (Implicit vs Explicit Dual Control).} \emph{Explicit dual control is less efficient than its implicit counterpart, even with fine tuning.}
\end{itemize}

\subsection{Objective and Awareness 
Uncertainty (Example 1)}
\label{sec:Ex1}

The performance-safety trade-off curve for Example 1 plotted in Fig.~\ref{fig:Ex1_tradeoff} validates H1.
Here, we design each planner with a set of fine-tuned parameters, e.g. weight of the collision avoidance cost and acceleration limits.
For a given planner design, we simulate the scenario 100 times, each with a different random seed, which affects three uncertainty sources: initial conditions, trajectories of the hidden states, and additive disturbances.
Although the ED-SMPC policy also manages to achieve a low collision rate thanks to its ability to actively reduce the human uncertainty, its overall closed-loop performance is consistently inferior to that of ID-SMPC, which validates H3.

\begin{figure}[!hbtp]
  \centering
  \includegraphics[width=1.0\columnwidth]{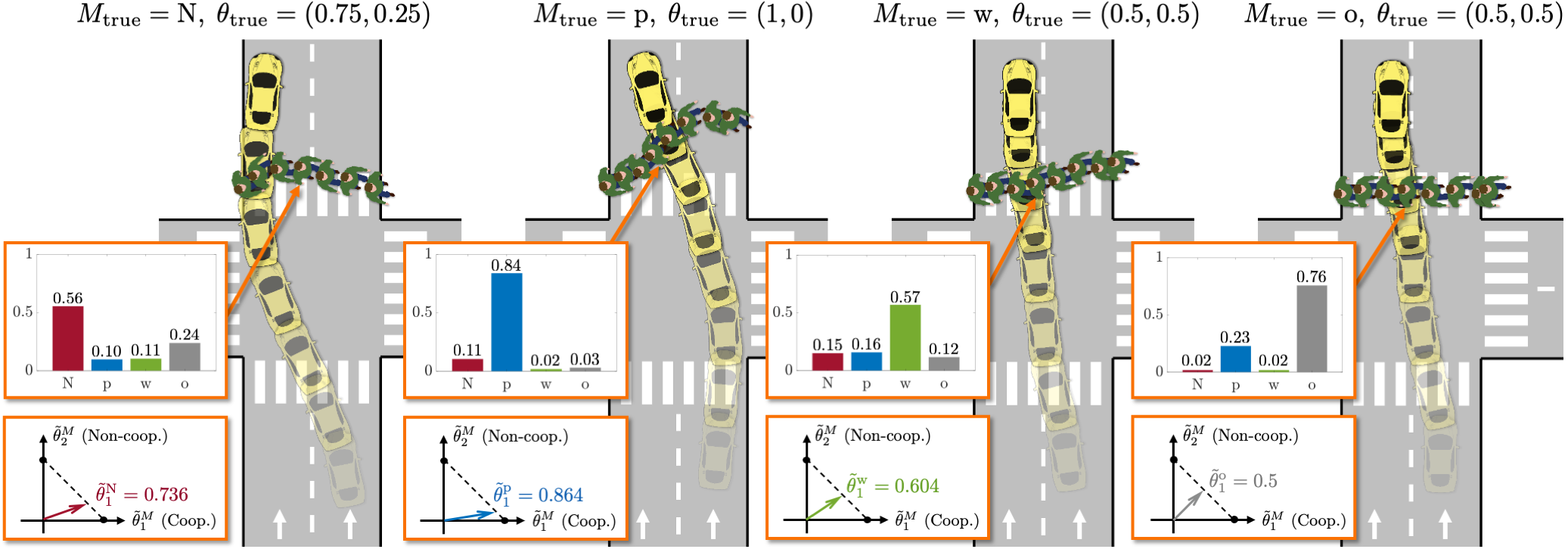}
  \caption{\label{fig:Ex2} Simulation snapshots of Example 2 using the proposed ID-SMPC planner. 
  The ground-truth hidden states are fixed throughout the simulation.
  In the orange blocks we display the robot's running belief $\distr(M \mid \ivec_{t})$ and maximum \emph{a posteriori} mean of $\theta^M$ with $t = 2.6$ s.
  The robot was able to quickly identify the human's hidden states and planned a collision-free trajectory in all four trials, accounting for the anticipated uncertainty reduction and interactions with the human.
  }
\end{figure}

Trajectory snapshots and evolution of $\prob(M=\text{l})$ of one simulation trial are shown in Fig.~\ref{fig:Ex1_Trajs}.
The priors are chosen as $P(M=\text{l})=1$ and $\theta^{\text{l}},\theta^{\text{r}} \sim \gaussian((0.5,0.5),5I)$.
Unlike non-dual control planners, ID-SMPC controlled the robot to approach the human-driven vehicle along the center of the road, allowing the robot to \new{informatively} probe the human---which resulted in a more accurate prediction of $M$ (bottom)---and guiding the robot through a region from which collisions can be avoided more easily.
Indeed, as the human's hidden state $M$ switched from $\text{r}$ to $\text{l}$ at $t = 6$~s, the robot using ID-SMPC executed a sharp right turn and successfully avoided colliding with the human.
The ED-SMPC planner, although effective at reducing the uncertainty at the beginning, failed to recognize that overtaking the human from the right would have resulted a more efficient trajectory.
All non-dual control planners, \emph{even with replanning}, caused a collision with the human due to insufficient knowledge about $M$.
It is also worth noticing that even if ID-SMPC uses ND-SMPC solutions for initialization, their closed-loop behaviors are vastly different, which essentially comes from the dual control effect.

In Fig.~\ref{fig:Ex1_safety_liveness}, we examine another simulation trial of Example 1.
Ground truth values of the hidden states are $\theta=1$ and $M=\text{l}$, respectively.
Using ID-SMPC, the robot was able to safely overtake the human-driven car in $6$ s.
However, using the ISA-iLQ planner, the robot failed to overtake the human within $10$ s, which is the simulation horizon.
Due to lack of dual control effort, the robot was stuck behind the human, \emph{unaware of the human's willingness to make room for the robot}.
Fig.~\ref{fig:Ex1_safety_liveness} also shows an unsafe trajectory generated with the CE-MPC planner.
Results shown by Fig.~\ref{fig:Ex1_Trajs} and Fig.~\ref{fig:Ex1_safety_liveness} demonstrate that with dual control effort, the robot gains better safety and liveness when interacting with the human, which supports our hypothesis H2. 
\begin{figure}[!hbtp]
  \centering
  \includegraphics[width=0.98\columnwidth]{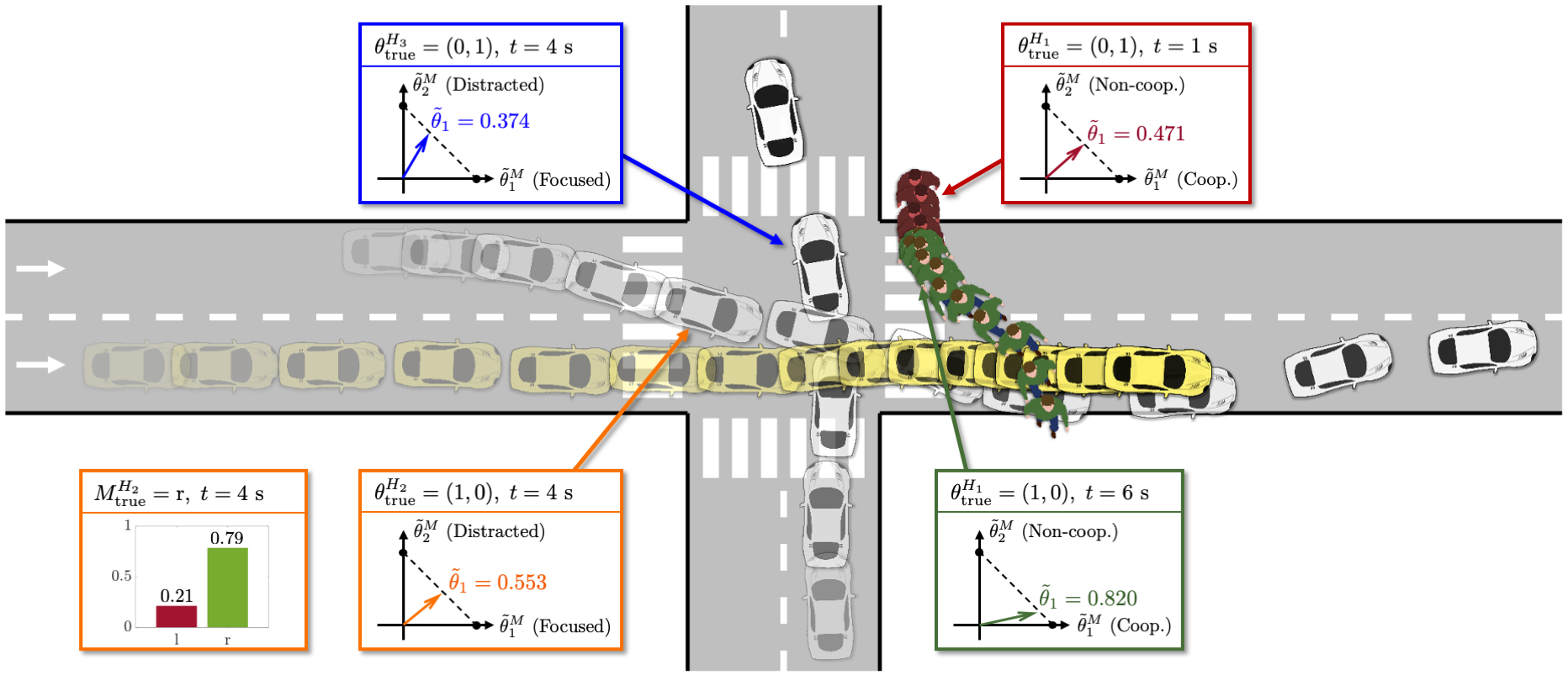}
  \caption{\label{fig:Ex3} Simulation snapshots and estimated hidden states of a multi-agent interaction scenario with a pedestrian ($H_1$) and two human-driven vehicles ($H_2$ and $H_3$) using the ID-SMPC planner.
  }
\end{figure}

\subsection{Behavioral and Cooperative Uncertainty (Example 2)}
\label{sec:Ex2}

Next, we consider an uncontrolled traffic intersection scenario with the human uncertainty introduced in Example 2.
Trajectory snapshots of four simulation trials with different hidden states are shown in Fig.~\ref{fig:Ex2}.
We chose uninformative prior distributions $\distr(M \mid \ivec_0)=\begin{bmatrix} 0.25 & 0.25 & 0.25 & 0.25 \end{bmatrix}$ and $\theta^{M} \sim \gaussian((0.5,0.5),5I)$ for all $M \in \mset$.

\subsection{Multi-Agent Case Study}
\label{sec:Ex3}

Finally, we apply ID-SMPC to the same intersection scenario as in Example 2 involving three human agents: two human-driven vehicles and a pedestrian.
The uncertainty for the human-driven vehicles and the pedestrian is modeled with Example 1 and 2, respectively.
Trajectory snapshots of one representative trial is shown in Fig.~\ref{fig:Ex3}.

\section{Discussion and Future Work}
\looseness=-1
We have introduced an implicit dual control approach towards active uncertainty reduction for human-robot interaction.
The resulting policy improves safety and efficiency of human-in-the-loop motion planning via a tractable approximation to the Bellman recursion
of a dual control problem,
automatically achieving an efficient
balance between improving expected performance and
eliciting information on future human behavior.
We see our work as an important step towards a broader class of methods that can handle general parametrizations of human behavior, including the effect of robot decisions on the human's hidden state.
While this paper focuses on the robot's own performance, our approach
may be adapted
to account for social coordination and altruism~\cite{toghi2021cooperative} in cooperative human-robot settings.
Finally, 
there is an open opportunity to derive
formal 
probabilistic or worst-case safety guarantees based on results established in~\cite{hu2022sharp,de2021scenario}.

\vspace{-0.2cm}
%
%
\printbibliography{}


\appendix

\section{Recursive Bayesian Estimation Updates}
\label{apdx:Bayes_est}

In this section, we write $B^R := B^R(\state_t)$, $B^H := B^H(\state_t)$, $U^M(\state_t, \ctrl^R_t) := U^M$ and $\covar^{\bar{\dstb}^M_t}(\state_t, \ctrl^R_t) := \covar^{\bar{\dstb}^M_t}$ for notational convenience.
Based on \eqref{eq:joint_sys_approx}, we can compute the likelihood of state $\state_{t+1}$ as
\begin{equation}
\label{eq:state_likelihood}
    \prob(\state_{t+1} \mid \ctrl^R_{t}, \ivec_{t}; \theta^M, M) = \textstyle\gaussian\left( B^H U^M \theta^M + \dyn (\state_t) + B^R \ctrl^R_t, \covar_{\bar{\dstb}^M_t}(\state_t, \ctrl^R_t)\right)
\end{equation}
Recall that we can decompose the belief state as $\bel_t = \prob\left(\theta^M \mid \ivec_{t}; M\right) \prob\left(M \mid \ivec_{t}\right)$.
Suppose the prior $\prob\left(\theta^M | \ivec_{t}; M\right) \sim \gaussian\left(\mean^{\theta^M}_{t},\covar^{\theta^M}_{t}\right)$ is given as a Gaussian distribution.
The conditional distribution $\distr(\theta^M_{-} \mid \ivec_{t+1}; M)$ can be computed by measurement update~\eqref{eq:Bayes_est_meas_theta} in closed-form and remains a Gaussian, whose mean and covariance are given by
\begin{equation*}
\begin{aligned}
    \mean^{\theta^M_{-}}_{t+1} &= \covar^{\theta^M_{-}}_{t+1} \left[ \left( B^H U^M \right)^\top {\covar^{\bar{\dstb}^M_t}}^{-1}\left( x_{t+1} - \dyn (\state_t) - B^R \ctrl^R_t \right) + \left(\covar^{\theta^M}_{t}\right)^{-1} \mean^{\theta^M}_{t} \right],\\
    \covar^{\theta^M_{-}}_{t+1} &= \left[ \left(\covar^{\theta^M}_{t}\right)^{-1} + \left( B^H U^M(\state_t, \ctrl^R_t) \right)^\top {\covar^{\bar{\dstb}^M_t}}^{-1} (\state_t, \ctrl^R_t) B^H U^M(\state_t, \ctrl^R_t)\right]^{-1}.
\end{aligned}
\end{equation*}
Subsequently, measurement update \eqref{eq:Bayes_est_meas_M} for $M$ can be readily computed by marginalizing state likelihood \eqref{eq:state_likelihood} with respect to $\theta^M$, similar to~\cite{arcari2020dual}.

\section{Scenario Tree Construction Procedure}
\label{apdx:tree}
\begin{algorithm}[H]
	\caption{Offline scenario tree construction}
	\label{alg:tree}
	\begin{algorithmic}[1]
	\Require Current state $\hat{x}_t$ and belief state $\hat{b}_t$, horizon $N > 0$, dual control horizon $1 \leq N^d \leq N$, mode set $\mset$, branching number $K > 0$
	\Ensure A scenario tree defined by node set $\nodeset_t$
	\State Initialization: $\state_{n_0} \gets \hat{\state}_t$,\;$\bel_{n_0} \gets \hat{b}_t$,\;$t_{n_0} \gets 0$,\;$\nodeset_t \gets \{n_0\}$
	\LineComment{Dual Control Steps:}
	\ForAll{$t^\prime \gets 0, 1, \ldots, N^d$}
        \ForAll{$\tnode \gets \nodeset_t$}
	        \If{$\tnode.t = t^\prime$}
	            \State Branch out child nodes: $\nodeset_t \gets \nodeset_t \cup \textsc{Branch}(\tilde{n},\mset,K)$
            \EndIf
	    \EndFor
    \EndFor
    \LineComment{Exploitation Steps:}
    \ForAll{$t^\prime \gets N^d, N^d+1, \ldots, N$}
        \ForAll{$\tnode \gets \nodeset_t$}
	        \If{$\tnode.t = t^\prime$}
	            \State Extend to get one child node: $\nodeset_t \gets \nodeset_t \cup \textsc{Extend}(\tilde{n})$
            \EndIf
	    \EndFor
    \EndFor
	\end{algorithmic}
\end{algorithm}

\begin{algorithm}[H]
	\caption{Generate child nodes for dual control steps}
	\label{alg:branch}
	\begin{algorithmic}[1]
	\Function \textsc{Branch}$(n, \mset, K)$
	\State Initialize the set of child nodes: $\cnodeset \gets \emptyset$
	\ForAll{$M \gets \mset$}
	    \State Randomly sample a set $\{\theta^{M,o}_1,\theta^{M,o}_2,\ldots,\theta^{M,o}_K\}$ from the standard Gaussian $\gaussian(0,I)$
	    \State Randomly sample a set $\{\bar{\dstb}^{M,o}_1,\bar{\dstb}^{M,o}_2,\ldots,\bar{\dstb}^{M,o}_K\}$ from the standard Gaussian $\gaussian(0,I)$
	    \ForAll{$k \gets 1,2,\ldots,K$}
	        \State Create a node $\tnode$
	        \State $t_\tnode = t_\node + 1$,\;$\theta^{M,o}_\tnode = \theta^{M,o}_k$,\;$\bar{\dstb}^{M,o}_\tnode = \bar{\dstb}^{M,o}_k$,\;$M_\tnode = M$
	        \State $\cnodeset \gets \cnodeset \cup \{\tnode\}$
	    \EndFor
	\EndFor
	\State \Return $\cnodeset$
	\end{algorithmic}
\end{algorithm}

\begin{algorithm}[H]
	\caption{Generate a child node for exploitation steps}
	\label{alg:extend}
	\begin{algorithmic}[1]
	\Function \textsc{Extend}$(n)$
	\State Create a node $\tnode$
	\State $t_\tnode = t_\node + 1$,\;$\theta^{M,o}_\tnode = 0$,\;$\bar{\dstb}^{M,o}_\tnode = 0$,\;$M_\tnode = M_\node$
	\State \Return $\{\tnode\}$
	\end{algorithmic}
\end{algorithm}

\section{Practical Aspects}
\label{sec:practical}

\subsection{Computation of the Human's Rational Actions}
The Laplace approximation used by~\eqref{eq:Laplace} requires the mean function (human's rational action) $\mu_i^M(\state_t, \ctrl^R_t)$ as the maximizer of $\qfunc_{i}^{M}\left(\state_{t}, \ctrl^R_t, \ctrl_{i}^{M}\right)$.
In our paper, we use the game-theoretic approach~\cite{fridovich2020efficient} to compute $\qfunc_{i}^{M}(\cdot)$ as an approximate local Nash equilibrium solution, which adopts an analytical maximizer.
In case when the expression of $\mu_i^M(\cdot)$ cannot be computed beforehand, we use a numerical approach similar to~\cite{sadigh2018planning}, which computes a \emph{local} maximizer $\mu_i^M(\cdot)$ during online optimization.
Under the mild assumption that $\qfunc_{i}^{M}(\cdot)$ is a smooth function whose maximum can be attained, we can set the gradient of $\qfunc_{i}^{M}(\cdot)$ with respect to $\ctrl_{i}^{M}$ to $0$.
This condition can be enforced either as a hard constraint or a penalty cost in ST-SMPC problem~\eqref{eq:ST-SMPC}.

\subsection{Handling Unbounded Support of Predicted Human's Action}
\looseness=-1
Constraint $\ctrl^H_{{\tnode}} \in \cset^H$ in~\eqref{eq:ST-SMPC} may not be feasible, since the predicted human's action $\ctrl^H_{{\tnode}}$
is given as a weighted sum of (unbounded) basis functions with (unbounded) normally distributed weights.
To reconcile this, we define, for each node ${\tnode}$, two separate decision variables: $\tilde{\ctrl}^H_{{\tnode}}$, which must equal the sampled linear combination of basis functions, and $\ctrl^H_{{\tnode}}$, which must satisfy $\ctrl^H_{{\tnode}} \in \cset^H$.
By adding a cost term $C\|\tilde{\ctrl}^H_{{\tnode}}-\ctrl^H_{{\tnode}}\|_2$ to~\eqref{eq:ST-SMPC}, with some large $C>0$ (we use $C=10^8$), the solver sets ${\ctrl}^H_{{\tnode}}$ to the nearest point in $\cset^H$ to the sample-consistent $\tilde{\ctrl}^H_{{\tnode}}$.
This feasible ``projected'' control ${\ctrl}^H_{{\tnode}}$ enters the dynamics in~\eqref{eq:ST-SMPC}.

\subsection{Parameter Estimation in Belief State Update}
Recall that in order to derive the approximate belief state dynamics $\tilde{g}(\cdot)$, we need to estimate the value of $\theta^M$ in~\eqref{eq:combined_dstb_covar}.
In our paper, we use a roll-out-based approach by setting the estimation $\bar{\theta}^M$ to the mean of the conditional distribution of $\theta^M$ computed in Step 3 of the initialization pipeline, which is described in the next section.

\subsection{Initialization Pipeline}
Since problem~\eqref{eq:ST-SMPC} is in general a large-scale nonconvex optimization problem, initialization is crucial for solving it rapidly and reliably in real-time.
In this paper, we generate an initial guess for~\eqref{eq:ST-SMPC} using the following pipeline:
\begin{enumerate}
    \item (Optional) Solve a certainty-equivalent MPC by setting $\theta^M$ and $M$ to their maximum a posteriori estimated values based on the current belief state $\hat{b}_t$.
    \item Solve a non-dual SMPC with the same scenario tree structure as the dual-SMPC, replacing belief state dynamics $\tilde{g}(\cdot)$ with $g^t(\cdot)$ for all dual control steps, and using the certainty-equivalent MPC solution as the initial guess.
    \item Forward-propagate belief states through $\tilde{g}(\cdot)$ using the non-dual SMPC solution for all dual control nodes in the scenario tree.
\end{enumerate}
Step 1 is optional and is only needed when the non-dual SMPC in Step 2 cannot be readily solved.
In Section~\ref{sec:sim}, we show that even if~\eqref{eq:ST-SMPC} uses results of the non-dual SMPC as its initialization, the resulting closed-loop trajectories are significantly different.

\end{document}